\documentclass{article}

     \PassOptionsToPackage{numbers, compress}{natbib}

     \usepackage[final]{neurips_2022}

\usepackage[utf8]{inputenc} 
\usepackage[T1]{fontenc}    
\usepackage{hyperref}       
\usepackage{url}            
\usepackage{booktabs}       
\usepackage{amsfonts}       
\usepackage{nicefrac}       
\usepackage{microtype}      
\usepackage{xcolor}         

\usepackage{amsmath}
\usepackage{amssymb}
\usepackage{mathtools}
\usepackage{amsthm}
\usepackage{dsfont}

\usepackage[capitalize,noabbrev]{cleveref}
\usepackage{adjustbox}
\usepackage{graphicx}

\usepackage{comment}
\newcommand{\specialcell}[2][c]{%
  \begin{tabular}[#1]{@{}c@{}}#2\end{tabular}}

\theoremstyle{plain}
\newtheorem{theorem}{Theorem}[section]

\newtheorem{lemma}[theorem]{Lemma}

\theoremstyle{definition}

\theoremstyle{remark}

\newcommand{\indep}{\perp \!\!\! \perp}

\title{Discovery of Single Independent Latent Variable}

\author{%
  Uri Shaham \\
  Department of Computer Science\\
  Bar-Ilan University \\
   Ramat Gan, Israel \\
  \texttt{uri.shaham@biu.ac.il} \\
   \And
   Jonathan Svirsky \\
   Faculty of Engineering\\
   Bar-Ilan University \\
   Ramat Gan, Israel \\
   \texttt{svirskj@biu.ac.il} \\
   \And
   Ori Katz \\
   Electrical and Computer Engineering\\
   Technion \\
   Haifa, Israel \\
   \texttt{orikats@campus.technion.ac.il} \\
   \And
   Ronen Talmon \\
   Electrical and Computer Engineering\\
   Technion \\
   Haifa, Israel \\
   \texttt{ronen@ee.technion.ac.il} \\
}

\begin{document}

\maketitle

\begin{abstract}
Latent variable discovery is a central problem in data analysis with a broad range of applications in applied science.
In this work, we consider data given as an invertible mixture of two statistically independent components, and assume that one of the components is observed while the other is hidden. Our goal is to recover the hidden component.
For this purpose, we propose an autoencoder equipped with a discriminator.
Unlike the standard nonlinear ICA problem, which was shown to be non-identifiable, in the  special case of ICA we consider here, we show that our approach can recover the component of interest up to entropy-preserving transformation.
We demonstrate the performance of the proposed approach in several tasks, including image synthesis, voice cloning, and fetal ECG extraction. 
\end{abstract}

\section{Introduction}

The recovery of hidden components in data is a long-standing problem in applied science. This problem dates back to the classical PCA \cite{pearson1901liii,hotelling1933analysis}, yet it has numerous modern manifestations and extensions, e.g., kernel-based methods \cite{scholkopf1998nonlinear}, source separation~\cite{comon2010handbook, belouchrani1997blind}, manifold learning \cite{tenenbaum2000global,roweis2000nonlinear,belkin2003laplacian,coifman2006diffusion}, and latent Dirichlet allocation \cite{blei2003latent}, to name but a few.
Perhaps the most relevant line of work in the context of this paper is independent component analysis (ICA) \cite{hyvarinen2000independent}, which attempts to decompose an observed mixture into statistically independent components.

Here, we consider the following ICA-related recovery problem. Assume that the data is generated as an invertible mixture of two (not necessarily one dimensional) independent components, and that one of the components is observed while the other is hidden. In this setting, our goal is to recover the latent component.
At first glance, this problem setting may seem specific and perhaps artificial. However, we posit that it is in fact broad and applies to many real-world problems.

For example, consider thorax and abdominal electrocardiogram (ECG) signals measured during labor for the purpose of determining the fetal heart activity. In analogy to our problem setting, the abdominal signal can be viewed as a mixture of the maternal and fetal heart activities, the maternal signal can be viewed as an accurate proxy of the maternal heart activity alone, and the fetal heart activity is the hidden component of interest we wish to recover.    
In another example from a different domain, consider a speech signal as a mixture of two independent components: the spoken text and the speaker identity. 
Arguably, the speaker identity is associated with the pitch and timbre, which are independent of information about the textual content, rhythm and volume.
Consequently, recovering a speaker-independent representation of the spoken text facilitates speech synthesis and voice conversion.

In this paper, we present an autoencoder-based approach, augmented with a discriminator, for this recovery problem. First, we theoretically show that this architecture of solution facilitates the recovery of the latent component up to an entropy-preserving transformation. Second, in addition to the recovery of the latent component (the so-called analysis task), it enables us to generate new mixtures corresponding to new instances of the observed independent component (the so-called synthesis task).  
Experimentally, we show both analysis and synthesis results on several datasets, consisting of simulated and real-world data. In particular, we demonstrate the proposed approach on ECG analysis, image synthesis, and voice cloning tasks.

Our contributions are as follows.
(i) We propose an easy-to-train mechanism for extraction of a single latent independent component. (ii) We present a simple proof for the ability of the proposed approach to recover the latent component. 
(iii) We experimentally demonstrate the applicability of the proposed approach in the contexts of both analysis and synthesis tasks. Specifically, we show applications to real-world data from different fields.



\section{Related Work}\label{sec:related}

The problem we consider in this work could be viewed as a simplified case of the classical formulation of nonlinear ICA. 
%
Several algorithms have been proposed for recovery of the independent components, assuming that (i) the mixing of the components is linear, and (ii) the components (with a possible exception of one) are non-Gaussian; this case was proven to be identifiable~\citep{comon1994independent, eriksson2004identifiability}.
The nonlinear case, however, i.e., when the mixing of the independent components is an arbitrary invertible function, was proven to be non-identifiable in the general case~\citep{hyvarinen1999nonlinear}.

{\bf{Identifiable nonlinear ICA.}}
%
\citet{hyvarinen2019nonlinear} have recently described a general framework for identifiable nonlinear ICA, generalizing several earlier identifiability results for time series, e.g.,~\citep{hyvarinen2016unsupervised, hyvarinen2017nonlinear, sprekeler2014extension}, in which in addition to observing the data $x$, the framework requires an auxiliary observed variable $u$, so that conditioned on $u$, the latent factors are independent (i.e., in contrast to being marginally independent as in a standard ICA setting). 
The approach we propose in this work falls into this general setting, as we assume that the auxiliary variable $u$ is in fact one of the latent factors, which immediately satisfies the conditional independence requirement. 
For this special case we provide a simple and intuitive recovery guarantee.

Following works have recently extended the framework of~\citet{hyvarinen2019nonlinear} to generative models~\citep{khemakhem2020variational}, unknown intrinsic problem dimension~\citep{sorrenson2020disentanglement}, and multiview setting~\citep{gresele2020incomplete}.
With respect to iVAE~\citep{khemakhem2020variational}, we allow for the recovery of high-dimensional components, whereas in iVAE, only one-dimensional components are recovered. In addition, our work presents several important differences: (i) we formulate our recovery guarantee result in terms of entropy-preserving map rather than statistical identifiability. (ii) It allows for a compact proof of the recovery guarantee. (iii) We present experiments on real-world data and comparisons to leading methods per application domain, whereas in~\citep{khemakhem2020variational}, the iVAE approach was mostly demonstrated on simulated data.

{\bf{Disentangled representation learning methods.}}
While the main interest in ICA has originally been for purposes of analysis (i.e., recovery of the independent sources from the observed data), the highly impressive achievements in deep generative modeling in recent years have drawn much interest also to the direction of data synthesis (e.g., images) from independent factors. In the research community, this direction is often termed learning of disentangled representations, i.e., representations in which modification of a single latent coordinate in the representation affects the synthesized data by manipulating a single perceptual factor in the observed data, leaving other factors unchanged.
In a similar fashion to the ICA case, the task of learning disentangled representations in the general case was proved to be non-identifiable~\citep{locatello2019challenging}.
Several methods for learning disentangled representations have been recently proposed, most of which are based on a variational autoencoder (VAE, ~\citet{kingma2013auto}) formulation, and decompositions of the VAE objective, for example~\citep{higgins2016beta, kim2018disentangling, chen2016infogan, chen2018isolating, kumar2017variational, burgess2018understanding, esmaeili2019structured}.
GAN-based approaches for disentanglement have been proposed as well~\citep{chen2016infogan, brakel2017learning}.



{\bf{Domain confusion.}} Our proposed approach is based on the ability to learn an encoding in which different conditions (i.e., states of the observed factor) are indistinguishable. 
Such a principle has been in wide use in the machine learning community for domain adaptation tasks~\citep{ganin2016domain}.
A popular means to achieve this approximate invariance is via discriminator networks, for example, in~\citep{lample2017fader, mor2018universal, shaham2018batch, nachmani2019unsupervised}, where the high-level mechanism is similar to the one proposed in this work, although the specifics are different.

We remark that while the algorithm we propose here has been presented before, e.g., in the context of singing voice conversion \cite{nachmani2019unsupervised}, to the best of our knowledge, it was not discussed in the context of latent independent component recovery, and no identifiability results were shown.

\section{Problem Formulation}
\label{sec:setup}
Let $X \in \mathbb{R}^d$ be a random variable. 
We assume that $X$ is generated via $X = f(S,T)$, where $f$ is an unknown invertible function of two arguments, and $S\in\mathcal{S}$ and $T\in\mathcal{T}$ are two random variables in arbitrary domains, satisfying $S \indep T$. We refer to $S$ as the \textit{unobserved source} that we wish to recover, to $T$ as the \textit{observed condition}, and to $X$ as the \emph{observed input}.

Let $\{(s_i, t_i) \}_{i=1}^n$ be $n$ realizations of $ S\times T$, and for all $i=1,\ldots, n$, let $x_i = f(s_i, t_i)$. Given only input-condition pairs, i.e., $\{(x_i, t_i) \}_{i=1}^n$, we state two goals.
First, from analysis standpoint, we aim to recover the realizations $s_1,\ldots s_n$ of the unobserved source $S$.
Second, from synthesis standpoint, for any (new) realization $t$ of $T$, we aim to to generate an instance $x = f(s_i,t)$ for any $i=1,\ldots,n$.



\section{Autoencoder Model with Discriminator}\label{sec:proposed}
To achieve the above goals, we propose an autoencoder (AE) with a discriminator. The AE model is denoted by $(E,D)$, where the encoder $E$ maps inputs $x$ to codes $s'$ (analysis), and the decoder $D$ maps code-condition tuples $(s', t)$ to input reconstructions $\hat{x}$ (synthesis), i.e., $x \xmapsto{E} s'$ and $(s',t) \xmapsto{D} \hat{x}$.
The discriminator, denoted by $g(\cdot)$, maps codes $s'$ to predicted conditions $\hat{t}$\footnote{This formulation of the discriminator does not capture all scenarios, however we use it here for simplicity. In \Cref{subsec:train_obj}, we describe additional implementations of the discriminator.}.
%


 We use two objective terms in a GAN-like minimax game: 
  \begin{equation}\label{eq:gan_minmax}
 \min_{E,D} \max_g  \left[ \mathcal{R}\left(x, D(E(x),t)\right) -\lambda \mathcal{I}\left(t, g(E(x))\right)\right].
\end{equation}
The first objective, denoted by $\mathcal{R}$, measures the discrepancy between an input $x$ and its reconstruction $\hat{x}=D(E(x),t)$. The second objective, denoted by $\mathcal{I}$, quantifies the independence between the condition $T$ and code $S'$ via a prediction $\hat{t}=g(s')$. 
In order to maximize $-\mathcal{I}\left(t, g(E(x))\right)$, the discriminator aims at leveraging any information in the code $S'=E(X)$ on the condition $T$ (via a learned function $g$) to obtain an accurate prediction $\hat{t}=g(s')$. 
In order to minimize both $\mathcal{R}\left(x, D(E(x),t)\right)$ and $-\mathcal{I}\left(t, g(E(x))\right)$, the autoencoder aims at reconstructing the input from the code $s'$ and the condition $t$, while failing the discriminator. We will show formally and empirically, that this results in an equilibrium in which $S'$ does not contain any information on $T$ and contains all the remaining information in $X$.
Our approach is illustrated in Figure~\ref{fig:proposed}.

\begin{figure}
  \centering
    \includegraphics[width=.55\columnwidth]{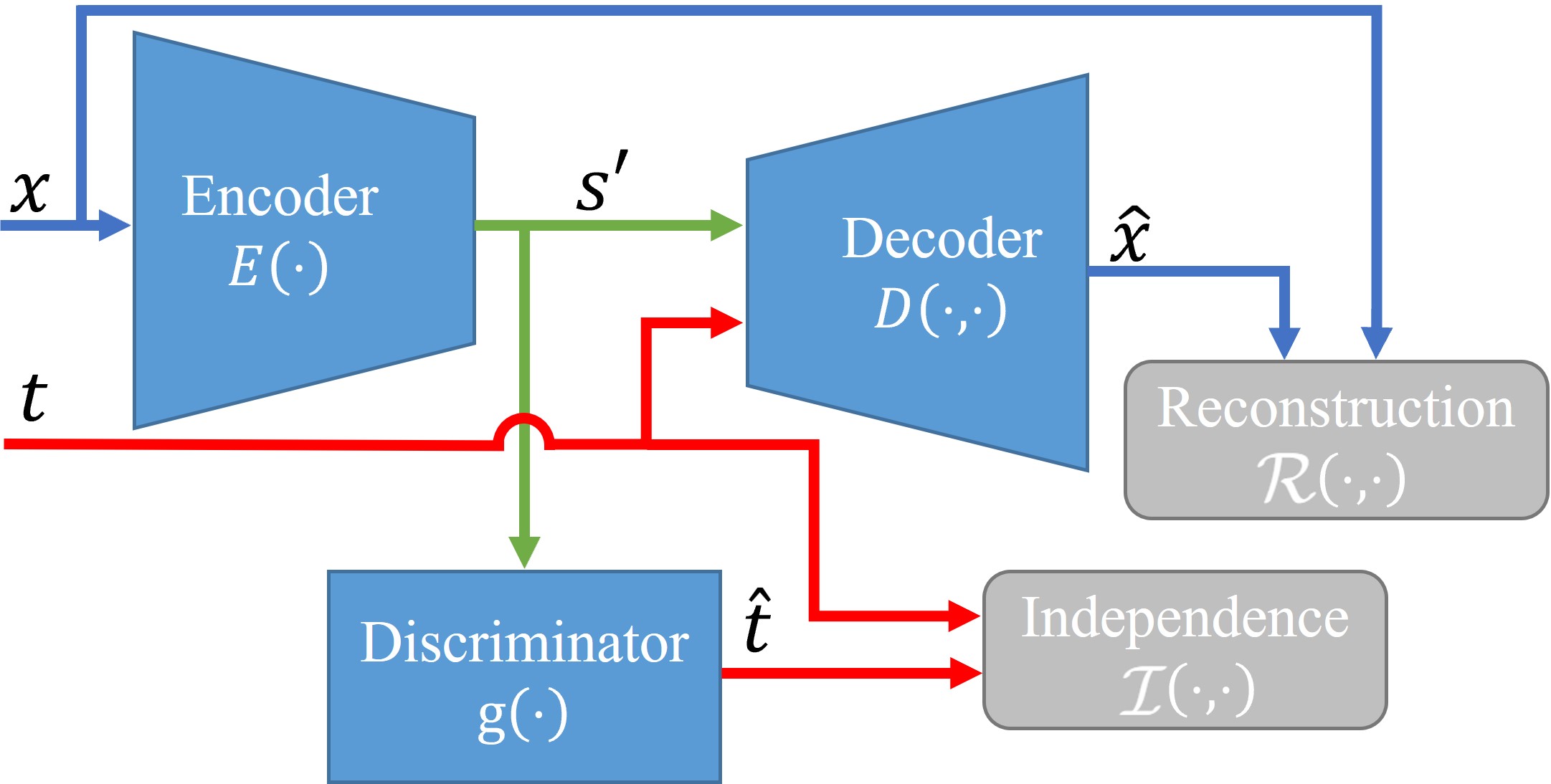}
    \caption{A diagram of the proposed approach. Learned functions are colored in blue, while the objective functions are colored in gray.}
    \label{fig:proposed}
\end{figure}


\subsection{Recovery of the Latent Component}\label{sec:identifiability}
Let $S'=E(X)$ be a random variable representing the encoder output, and let $\hat{X} = D(S', T)$ be a random variable representing the decoder output.

Lemma~\ref{lemma} establishes that when the autoencoder is trained to perfect reconstruction and the learned code is independent of the condition, the learned code contains the same information as $S$, thereby proving that the latent component of interest can be recovered, up to an entropy-preserving transformation.
The lemma is stated assuming $S$ is discrete and in terms of mutual information $I(\cdot ; \cdot)$ and entropy $H(Y) := -\sum_{y \in \mathcal{Y}} p(y)\log p(y)$, where $Y$ is a random variable with density $p$, taking values in $\mathcal{Y}$.
An equivalent result for the case of continuous $S$ can be obtained by replacing the entropy term with the limiting density of discrete points~\cite{jaynes1957information} $H(Y) := -\int_{y \in \mathcal{Y}} p(y)\log \frac{p(y)}{m(y)}dy$, where $m(\cdot)$ is the limiting density.

\begin{lemma}\label{lemma}
Suppose we train the autoencoder model to zero generalization loss, i.e., $\hat{X} = X$, and impose on the code that $S' \indep T$. 
Then $I(S;S'|T) = H(S) = H(S')$.
\end{lemma}
\begin{proof}
Observe that since $f$ is invertible
\begin{align}
    \label{eq:x_inv}
    H(X) = H(S,T),
\end{align}
and there exist functions $h_1,h_2$ such that $h_1(X) = S,\; h_2(X) = T$. Since $S\indep T$ we have that $H(S,T) = H(S) + H(T)$, combining with \cref{eq:x_inv} we get that
\begin{align}
    \label{eq:x_fac}
    H(X) = H(S) +H(T).
\end{align}

Since $X = \hat{X} = D(S',T)$ we have that
\begin{align}
    \label{eq:x_leq}
    H(X) \leq H(S',T)= H(S')+H(T),
\end{align}
where the second transition is due to the fact that $S' \indep T$.
Combining \cref{eq:x_fac} with \cref{eq:x_leq} we get that 
\begin{align}
    \label{eq:ss'_leq}
    H(S) \leq H(S').
\end{align}

Next, observe that $(S',T)$ is a function of $X$: $(S',T) = \left(E(X), h_2(X)\right)$, hence
\begin{align}
    \label{eq:x_geq}
    H(X) \geq H(S',T)= H(S')+H(T),
\end{align}
where the second transition is due to the fact that $S' \indep T$.
Combining with \cref{eq:x_fac} we get that 
\begin{align}
    \label{eq:ss'_geq}
    H(S') \leq H(S).
\end{align}
Combining \cref{eq:ss'_leq} with \cref{eq:ss'_geq} we get that $H(S) = H(S')$.

Finally, following the definition of conditional mutual information as $I(X;Y|Z) = I(X;(Y,Z)) - I(X;Z)$ and setting $X=S, Y = S', Z = T$ we get
\begin{align}
I(S;S'|T) &= I(S;(S',T)) - I(S;T) \notag\\
&= I(S;(S',T)) \notag \\
&= H(S',T) - H(S',T|S) \notag \\
&= H(S',T) - H(S', T, S) + H(S)\notag \\
&= H(S',T) - H(S', T) + H(S)\notag \\
& = H(S),\notag
\end{align}
where the penultimate transition is due to the fact that $S = h_1(D(S',T))$.
\end{proof}

Lemma~\ref{lemma} has two important consequences. First, it shows that unlike the standard nonlinear ICA problem, the problem we consider here allows for recovery of the latent independent component of interest. More specifically, it proves that when the autoencoder yields perfect reconstruction and condition-independent code, the learned code is a recovery of the random variable $S$, up to entropy-preserving transformation.
Second, the lemma prescribes a recipe for the practical solution we present here.
Requiring the autoencoder to generate accurate reconstruction $\hat{X}$ of $X$ ensures that no information on $S$ is lost in the encoding process. 
Independence of $S'$ and $T$ is achieved implicitly; it results from the equilibrium of the GAN-like minimax game \eqref{eq:gan_minmax}, as the discriminator can benefit from any mutual information between $S'$ and $T$.

\subsection{Training Objectives}\label{subsec:train_obj}
As described above, to obtain a code $S'$ that is independent of the condition $T$, we utilize a discriminator network, aiming to leverage information on $T$ in $S'$ for prediction, and train the encoder $E$ to fail the discriminator, in a standard adversarial training fashion. Doing so pushes the learned codes $S'$ towards being a condition-free encoding of $X$.

We propose to optimize the following objectives of the discriminator and the autoencoder:
\begin{align}
    \label{eq:objecties}
    \mathcal{L}_\text{disc} &= \min_g \mathcal{I}\left(t, g(E(x))\right),\\
    \mathcal{L}_\text{AE} &= \min_{E,D} \left[ \mathcal{R}\left(x, D(E(x),t)\right) -\lambda \mathcal{I}\left(t, g(E(x))\right)\right].
\end{align}
%
The specific reconstruction and independence objective terms are application-dependent. In our experiments, we make use of the following.

{\bf{Reconstruction.}}
We use standard reconstruction loss functions. In the experiments with images, $\ell_1$ and SSIM loss~\cite{wang2004image} (and combinations of these) are used. $\ell_1$ loss is also used in the audio experiments, and MSE loss in the experiments with ECG signals. 

{\bf{Independence.}}
The discriminator computes a map $\hat{t}=g(s')$, where $s'=E(x)$ is the code obtained from the encoder and $\hat{t}=g(s')$ is the condition predicted by the discriminator. The discriminator is trained to minimize the independence term $\mathcal{I}(\hat{t}, t)$ (thus to leverage mutual information in $S'$ and $T$).

When the condition takes values from a finite symbolic set, we train the discriminator as a classifier that predicts the condition class from the code, and we set the independence term to $\mathcal{I}(\hat{t}, t) = \text{Cross Entropy}(\hat{t}, t)$.
This is also known as a Domain Confusion term.
Using this term, the autoencoder is trained to produce codes that maximize the cross entropy with respect to the true condition, and the equilibrium of the game is when this term equals the cross entropy of a random guess.

When the condition and its prediction take numerical values, i.e., $t,\hat{t} \in \mathbb{R}$, we train the discriminator as a regression model and set the independence term to:
     $\mathcal{I}(\hat{t}, t)=-\text{Correl}^2(\hat{t}, t)$.
We remark that this term also equals the negative $R^2$ term of a simple regression model, regressing $t$ on $\hat{t}$. Using this term, the autoencoder is trained to produce codes for which the squared correlation with the condition is minimized. As $\hat{t}$ is a nonlinear function of $s'$ computed via a flexible model such as a neural net, $\left(\text{Correl}(\hat{t}, t)\right)^2 = 0$ implies that $S'$ and $T$ are approximately statistically independent.

In addition, we also successfully train the discriminator in a contrastive fashion, i.e., to distinguish between ``true tuples'' $(s', t)$ that correspond to samples $(s,t)$ from the joint distribution of $S$ and $T$ satisfying $s'=E(x)$ and $x=f(s,t)$, and ``fake tuples'' $(s', t)$, where $s'=E(x)$ but $x=f(s,\tilde{t})$ with $t\neq \tilde{t}$. The contrastive objective is then $\mathcal{I}(s', t) = \text{Cross Entropy}(\hat{l}, l)$,
where $l$ is the ground truth true/fake label and $\hat{l} = g(s', t)$ is the predicted true/fake label made by the discriminator. Note that this implementation deviates from the discriminator formulation we used thus far. Here, the discriminator $g(\cdot,\cdot)$ maps tuples of codes and conditions $(s',t)$ to true/fake labels.

We remark that other possible implementations of the independence criterion can be utilized as well, e.g., nonlinear CCA~\cite{andrew2013deep, michaeli2016nonparametric} and the Hilbert-Schmidt Independence Criterion (HSIC)~\citep{gretton2005measuring}. 

{\bf{Optimizers.}}
Our proposed approach utilizes two optimizers, one for the autoencoder and one for the discriminator. The AE optimizer optimizes $\mathcal{L}_\text{AE}$ by tuning the encoder and decoder weights. The discriminator optimizer optimizes $\mathcal{L}_\text{disc}$ by tuning the discriminator weights (which determine the function $g$). A common practice in training GANs is to call the two optimizers with different frequencies. We specify the specific choices used in our experiments in Appendix~\ref{app:tech}.

{\bf{GAN real/ fake discriminator.}} Optionally, a GAN-like real / fake discriminator can be added as an additional discriminator in order to encourage generating more realistic inputs.
While we have a successful empirical experience with such GAN discriminators (e.g., see Appendix~\ref{app:cycle}), this is not a core requirement of our proposed approach.

\section{Experimental Results}\label{sec:experiments}
In this section, we demonstrate the efficacy of the proposed approach in various settings, by reporting experimental results obtained on different data modalities and condition types, in both analysis and synthesis tasks. We present four applications here and additional two in the appendix.

We begin with a two dimensional analysis demonstration, in which the condition is real-valued.
Second, we demonstrate the utility of our approach for image manipulation, where the condition is given as an image.
Third, the proposed approach is used for voice cloning, which is primarily a synthesis task with a symbolic condition. 
Fourth, we apply our approach to an ECG analysis task, using a real-valued heartbeat signal as the condition.
Additional experimental results in image synthesis are described Appendix~\ref{app:mnist} and~\ref{app:cycle}.

The network architectures and training hyperparameters used in each of the experiments are described Appendix~\ref{app:tech}. In addition, codes reproducing some of the results in this manuscript are available at~\url{https://github.com/shaham-lab/disilv}.

\subsection{2D Analysis Demonstration}
In this example, we first generate the latent representation of the data by sampling from two independent uniform random variables. 
We then generate the observed data via linear mixing. We consider one of the latent components as the condition and train the autoencoder to reconstruct the observed data, while obtaining code which is independent of the condition using the regression objective.
We use $\ell_1$ as a reconstruction term.
The top row in Figure~\ref{fig:2d} shows the latent, observed and reconstructed data, as well as the distribution of the condition and the learned code. 
The bottom row in Figure~\ref{fig:2d} shows the results of a similar setup, except for the mixing which is now nonlinear. As can be seen, the joint distribution of the learned code and the condition is approximately a tensor product of the marginal distributions, which implies that the latent component is indeed recovered.

\begin{figure}[t]
  \centering
    \includegraphics[width=.75\textwidth]{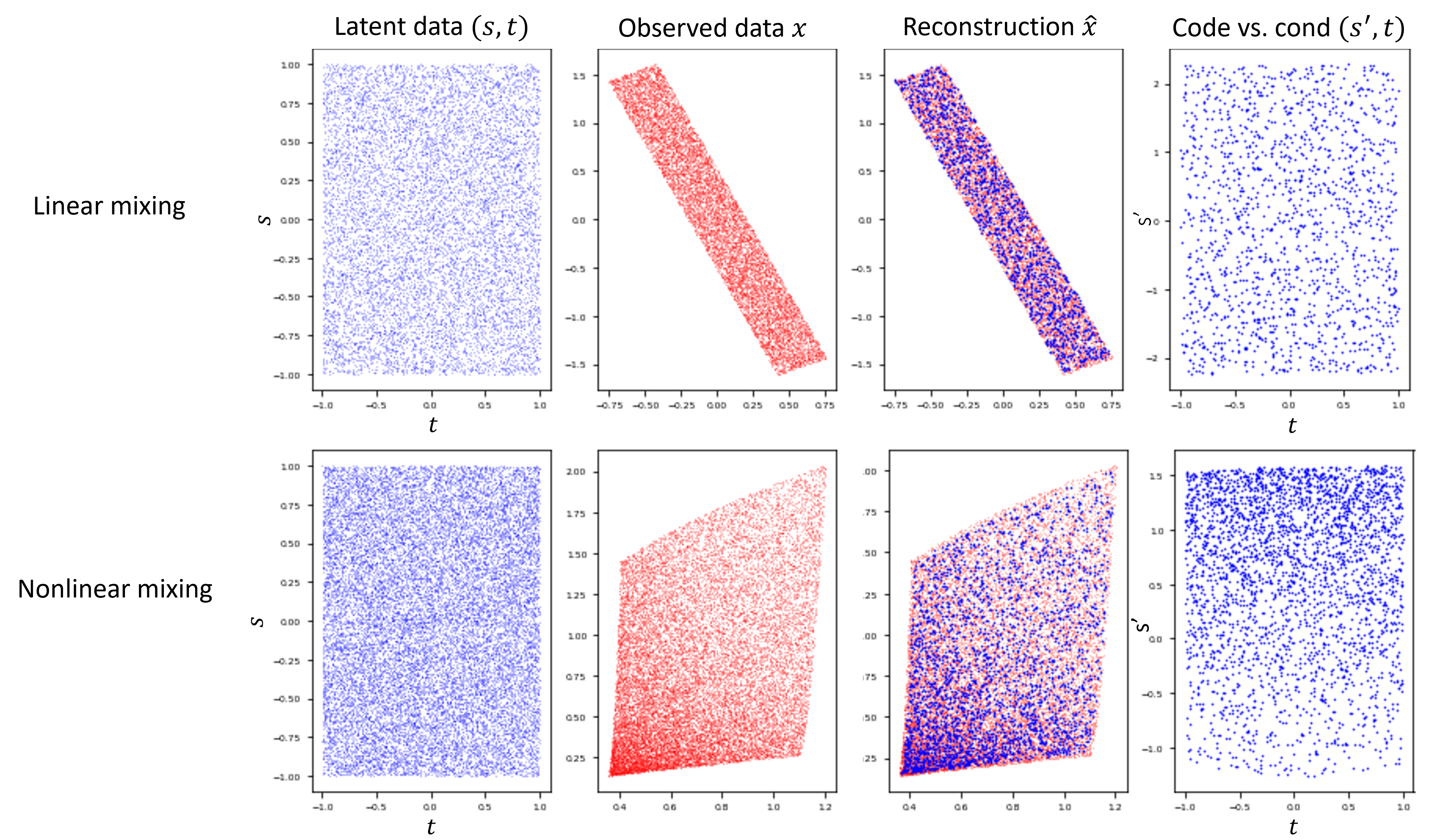}
    \caption{Analysis demonstration. The reconstruction plots show train data in {\color{red} red} and reconstructed test data in {\color{blue} blue}. The learned code $s'$ is a recovery of the latent factor $s$, up to entropy-preserving transformation (e.g., an arbitrary monotonic transformation). The approximate independence of the $S'$ and the condition $T$ can be recognized by noticing that the joint density of the code and the condition is an outer product of the marginal distributions. The $p$-values of $\chi^2$ independence test for the shown results are 0.75 (linear mixing) and 0.83 (nonlinear mixing)}
    \label{fig:2d}
\end{figure}

\subsection{Rotating Figures}

In this experiment, we use the setup shown in Figure~\ref{fig:bulldog_setup}, in which two figures, Bulldog and Bunny, rotate on discs. The rotation speeds are different and are not an integer multiple one of the other. 
The figures are recorded by two static cameras, where the right camera captures both Bunny and Bulldog, while the left camera captures only Bulldog.
The cameras operate simultaneously, so that in each pair of images Bulldog's position with respect to the table is the same. 
This dataset was curated in~\cite{lederman2018learning}.

We consider images from the right camera (which contain both figures) as the observed input $x$, and the images from the left camera (which only show Bulldog) as the condition $t$.
Note that the input can be considered as generated from two independent sources, namely the rotation angles of Bulldog and Bunny. 
The goal is to use $x$ and $t$ to recover the rotation angle $s$ of Bunny\footnote{A related work on this dataset was done in~\citep{shaham2018learning}, although there the goal was the opposite one, i.e., to recover the common information of the two views, which is the rotation angle of Bulldog.}.

Once training is done, we use the autoencoder to generate new images by manipulating Bulldog's rotation angle while preserving Bunny's. This is done by feeding $x$ to the encoder, obtaining an encoding $s'$, sampling an arbitrary condition $\tilde{t}$ and feeding $(s',\tilde{t})$ through the decoder. 
We use $\ell_1$ loss for reconstruction, and contrastive loss 
to train the discriminator. Namely, we train the discriminator to distinguish between (image, condition) tuples, which were shot at the same time, and tuples which were not. 
Figure~\ref{fig:rotating} shows an exemplifying result.
As can be seen, the learned model disentangles the rotation angles of Bunny and Bulldog and generates images in which Bunny's rotation angle is preserved while Bulldog's is manipulated.

\begin{figure*}[t]
  \centering
    \includegraphics[width=.71\textwidth]{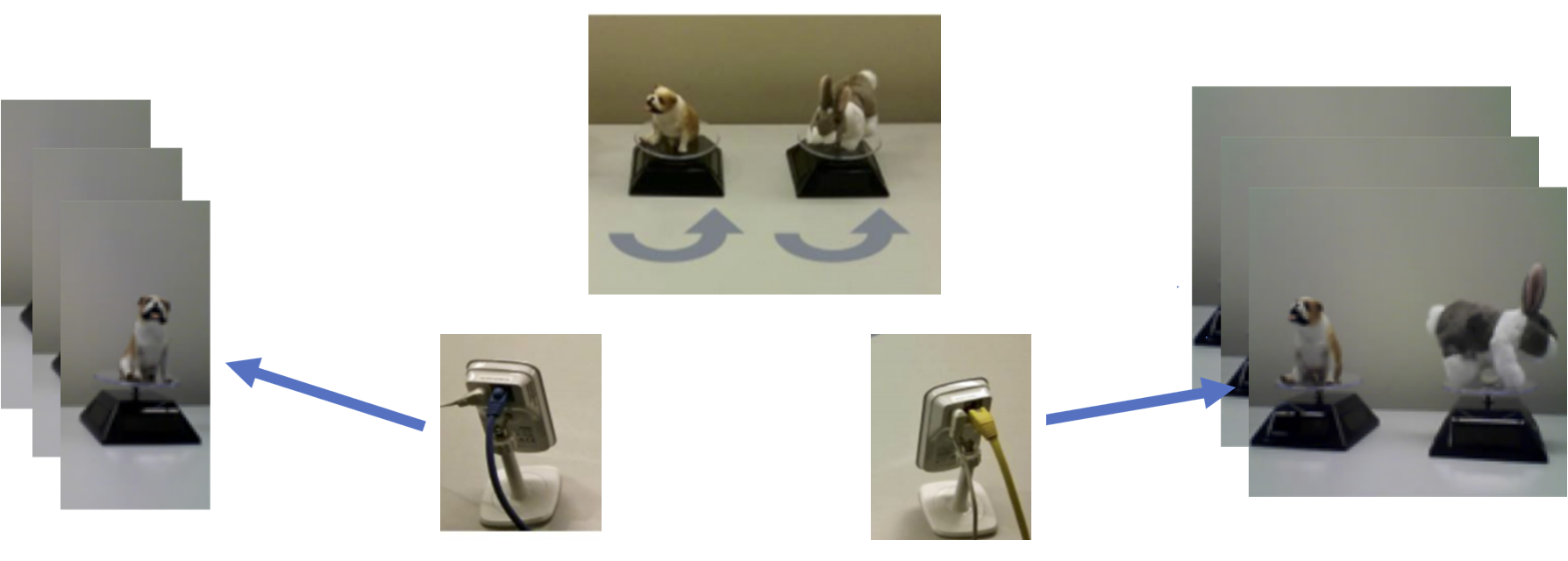}
    \caption{Experiment setup of the rotating figures. Bulldog and Bunny rotate in different speeds. The right camera captures both Bulldog and Bunny, while the left camera captures only Bulldog. Images from the right camera are considered as the input $x$, which is generated from two independent factors -- the rotation angles of the figures. Bulldog is considered as the condition $t$. The goal is to recover the rotation angle of Bunny, and to manipulate a given input image $x$ by plugging in a different condition than the one present in the image. The fact that $t$ and $x$ are captured from two different viewpoints prevents modification of the image simply by pasting Bulldog into $x$.} 
    \label{fig:bulldog_setup}
\end{figure*}

\begin{figure*}[t]
  \centering
    \includegraphics[width=.81\textwidth]{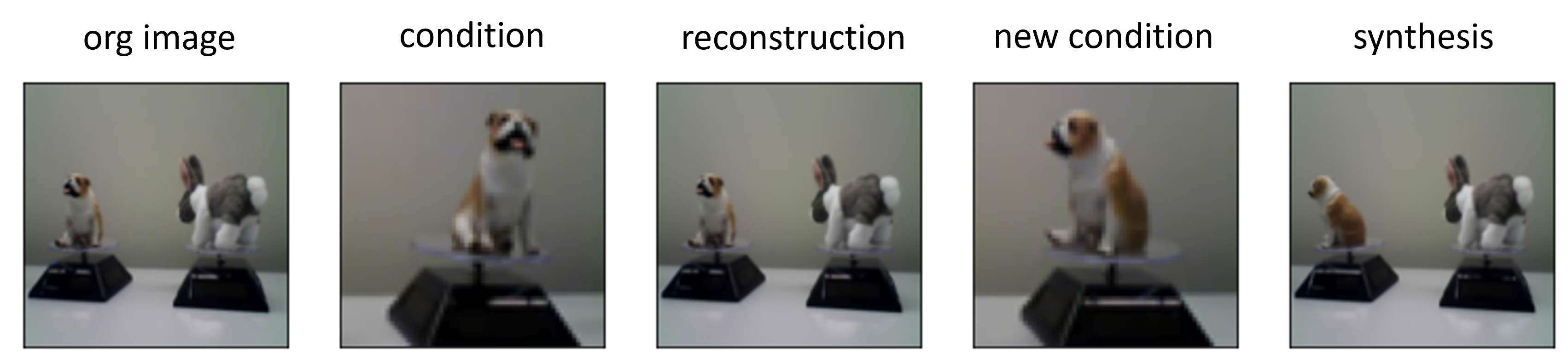} 
    \caption{The rotating figures experiment. From left to right: (i) the input image to the encoder $x$, (ii) the condition $t$ corresponding to Bulldog in $x$ as captured from the viewpoint of the left camera, (iii) the reconstruction $\hat{x}$ of $x$, (iv) a new condition plugged into the decoder, and (v) the resulting manipulated image.} 
    \label{fig:rotating}
\end{figure*}

\subsection{Voice Cloning}

To demonstrate the application of the proposed method to voice conversion, we run experiments on a non-parallel corpus, CSTR
VCTK Corpus~\citep{veaux2017cstr}, which consists of 109 English speakers
with several accents (e.g., English, American, Scottish, Irish,
Indian, etc.). 
In our experiments, we use a subset of the corpus
containing all the utterances for the first 30 speakers (p225-
p256, without p235 and p242). 

We construct the autoencoder to operate on Mel spectrograms using the speaker id as the condition. The AE architecture was based on Jasper~\citep{li2019jasper} blocks (specific details can be found in Appendix~\ref{app:tech}). The decoder uses a learnable lookup table with 64-dimensional embedding for each speaker. For the discriminator, we use the same architecture as in~\citep{mor2018universal}.
We use $\ell_1$ loss for reconstruction, and the discriminator is trained using domain confusion loss. 
Along with the reconstruction loss, in this experiment we also train a real/fake discriminator, applied to the output of the decoder.
To convert the decoder output to waveform, we use a pre-trained melgan~\cite{kumar2019melgan} vocoder.

Once the autoencoder is trained, we apply it to convert speech from any speaker to any other speaker. Some samples of the converted speech are given at \url{https://shaham-lab.github.io/disilv/}.
To evaluate the similarity of the converted voice and the target voice, we use MCD (Mel Cepstral Distortion) on a subset of the data containing parallel sentences of multiple speakers. Specifically, MCD computes the $\ell_1$ difference between dynamically time warped instance of the converted source voice and a parallel instance of the target voice, which is a common evaluation metric for voice cloning. We remark that the parallel data are used only for evaluation and not for training the model.
We use the script provided in~\citep{li2020espnet} to compute the MCD and compare our proposed approach to~\cite{ding2019group, polyak2019tts} and references therein, which are all considered to be strong baselines, trained on the VCTK dataset as well.
The results, shown in Table~\ref{tab:mcd}, demonstrate that our approach outperforms these strong baselines.

\begin{table}[t]
    \caption{Voice cloning results: Mel Cepstral Distortion (MCD) in terms of mean (std). PPG, PPG2 results are taken from~\cite{polyak2019tts}, VQ AVE and PPG GMM results are taken from~\cite{ding2019group}.}
    \label{tab:mcd}
    \begin{center}
    \begin{small}
    \begin{adjustbox}{width=\columnwidth,center}
    \begin{sc}
    \begin{tabular}{l | c c c c c c c}
        \toprule
         Method & TTS Skins~\cite{polyak2019tts} &  GLE~\cite{ding2019group} & VQ VAE & PPG GMM & PPG & PPG2 & Ours \\ 
         \midrule
         MCD &
         8.76  (1.72) &
         7.56 &
         8.43 &
         8.57 &
         9.19 (1.50) &
         9.18 (1.52) &
         \textbf {6.27} (1.44)\\
         \bottomrule
    \end{tabular}
    \end{sc}
    \end{adjustbox}
    \end{small}
    \end{center}
\end{table}

\subsection{Fetal ECG extraction}
\label{sec:ecg_main}
In this experiment, we demonstrate the applicability of the proposed approach to non-invasive fetal electrocardiogram (fECG) extraction, which facilitates the important task of monitoring the fetal cardiac activity during pregnancy and labor. 
%
Following commonly-used non-invasive methods, we consider extraction of the fECG based on two signals: (i) multi-channel abdominal ECG recordings, which consist of a mixture of the desired fECG and the masking maternal electrocardiogram (mECG), and (ii) thorax ECG recordings, which are assumed to contain only the mECG.
In analogy to our problem formulation (see Section~\ref{sec:setup}), the desired unobserved source $s$ denotes the fECG, the observed condition $t$ denotes the (thorax) mECG, and the input $x$ denotes the abdominal ECG.

{\bf{Dataset.}}
We consider the dataset from \cite{sulas2021non}, which is publicly available\footnote{\url{https://physionet.org/content/ninfea/1.0.0/}} on PhysioNet \cite{goldberger2000physiobank}.
This dataset was recently published and is part of an ongoing effort to establish a benchmark for non-invasive fECG extraction methods. 
%
The dataset consists of ECG recordings from $60$ subjects. Each recording consists of $n_a=24$ abdominal ECG channels and $n_t=3$ thorax ECG channels. In addition, it contains a pulse-wave doppler recording of the fetal heart that serves as a ground-truth. 
See Appendix~\ref{sec:appendix_ecg} for more details.

{\bf{Model training.}}
The input-condition pairs $(x_i, t_i)$ are time-segments of the abdominal ECG recordings ($x_i \in \mathbb{R}^{n_a \times n_T}$) and the thorax ECG recordings ($t_i \in \mathbb{R}^{n_t \times n_T}$), where the length of the time-segments is set to $n_T=2,000$ (4 seconds).
We train a separate model for each subject based on a collection of $n$ input-condition pairs $\{(x_i, t_i)\}_{i=1}^{n}$ of time-segments.

The encoder is based on a convolutional neural network (CNN), so that the obtained codes $s'_i=E(x_i) \in \mathbb{R}^{n_d \times n_T}$ are time-segments, where the dimension of the code is set to $n_d=5$.
For more details on the architecture, model training, and hyperparameters selection, see Appendix~\ref{sec:appendix_ecg}. 

We note that the training is performed in an unsupervised manner, i.e., we use the ground-truth doppler signal only for evaluation and not during training.

\begin{table*}[tb]
    \caption{fECG extraction results. In the leftmost column, we present $R_x$, and in the other columns we present $R_{s'}$ achieved by the different methods.}
    \label{tab:ResultsPerMethod}
    \vskip 0.15in
    \begin{center}
    \begin{small}
    \begin{sc}
    \scalebox{0.9}{
    \begin{tabular}{lcccccc}
    \toprule
     $\#$ of Subject & Input       & Ours& ADALINE     & ESN & LMS & RLS  \\ 
    \midrule
     Top 5  & 2.23 (3.23) & \textbf{6.86} (1.98) & 6.46 (2.54) & 1.99 (1.08) & 2.60 (1.60) & 1.03 (0.70) \\
     Top 10 & 1.20 (2.41) & \textbf{5.43} (2.02) & 4.22 (2.94) & 1.19 (1.10) & 1.56 (1.53) & 0.75 (0.56) \\
     Top 20 & 0.66 (1.75) & \textbf{3.53} (2.44) & 2.59 (2.63) & 0.71 (0.91) & 0.89 (1.26) & 0.51 (0.46) \\
     All    & 0.30 (1.17) & \textbf{1.84} (2.16) & 1.32 (2.08) & 0.36 (0.68) & 0.40 (0.94) & 0.26 (0.38) \\
     \bottomrule
    \end{tabular}
    }
    \end{sc}
    \end{small}
    \end{center}
\end{table*}

{\bf{Qualitative evaluation.}}
\label{par:qualitative}
In Figure~\ref{fig:raw_example} we present an example of an input-condition pair $(x_i, t_i)$ and the obtained code $s'_i=E(x_i, t_i)$.
\begin{figure}[t]
  \centering
    \includegraphics[width=.8\textwidth]{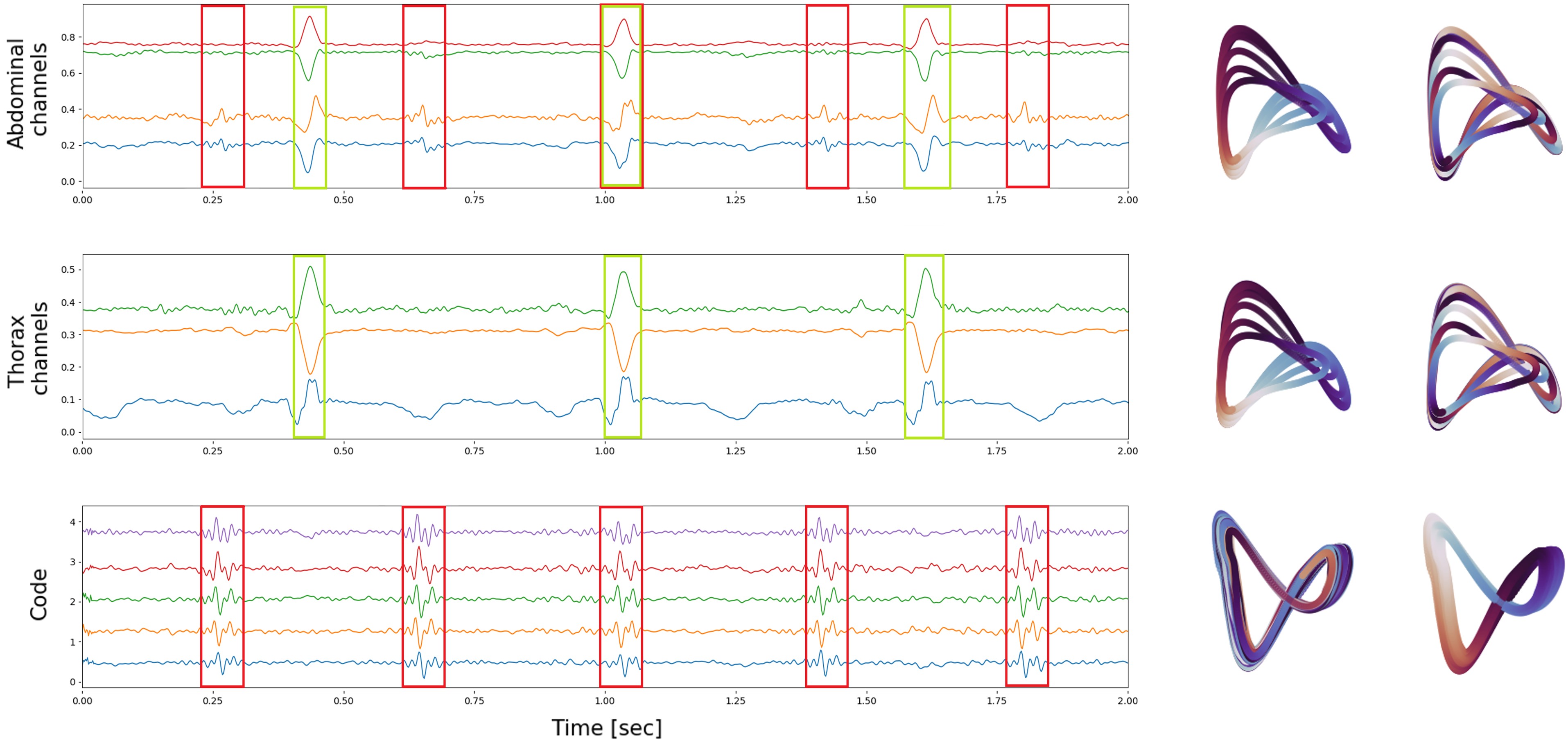}
    \caption{Example of an input-condition pair $(x_i, t_i)$ and the obtained code $s'_i$. The duration of the presented time segment is 2 sec. 
    Leftmost column: the raw channels. Middle column: the PCA embedding of the samples colored by the mECG. Rightmost column: the PCA embedding of the samples colored by the fECG. Top row: the abdominal channels $x_i$ (for brevity only 5 channels are presented). Middle row: the thorax channels $t_i$. Bottom row: the obtained code $s'_i$.}
    \label{fig:raw_example}
    \label{fig:time_trajectories}
    \vskip -0.2in
\end{figure}
We see that the abdominal channels consist of a mixture of the fECG and the mECG, where the fECG is significantly less dominant than the mECG and might even be completely absent from some of the channels. In addition, we see that the thorax channels are affected by the mECG only. Lastly, we see that the obtained code captures the fECG without any noticeable trace of the mECG.
In addition, we present the projections of $1,000$ sequentially-sampled inputs $x_i$ (abdominal channels), conditions $t_i$ (thorax channels), and their codes $s_i'=E(x_i)$ on their respective $3$ principal components.
We color the projected points by the periodicity of the mECG (middle column), computed from the thorax channels, and by the periodicity of the fECG (rightmost column), computed from the ground-truth doppler signals.
We see that the PCA of the abdominal and thorax channels are similar, implying that the mECG dominates the mixture.
In addition, we see that the color of the PCA of the abdominal and thorax channels according to the mECG (middle column) is similar and smooth, unlike the color by the fECG (rightmost column). In contrast, the PCA of the code is different (bottom row) and only the color by the fECG is smooth (rightmost column), indicating that the code captures the fECG without a significant trace of the mECG, as desired.

{\bf{Baselines.}}
%
We consider four baselines taken from a recent review \cite{kahankova2019review}. 
Specifically, we focus on methods that utilize reference thorax channels. 
The first two baselines are based on adaptive filtering, which is considered to be the traditional approach for fECG extraction: least mean squares (LMS) and recursive least squares (RLS). 
This approach was first introduced by \citet{widrow1975adaptive}, and it is still considered to be relevant in recent studies  \cite{martinek2017non,wu2013research,swarnalath2010novel}.
The third baseline is ADALINE~\cite{kahankova2018adaptive} which utilizes neural networks adaptable to the nonlinear time-varying properties of the ECG signal.
The fourth baseline is based on an echo state network (ESN) \cite{jaeger2001echo}. 


{\bf{Quantitative evaluation.}}
\label{par:quantitative}
To the best of our knowledge, there is no gold-standard nor definitive evaluation metrics for fECG extraction.
Here, based on the ground-truth doppler signal, we quantify the enhancement of the fECG and the suppression of the mECG as follows.

First, we compute the principal component of the input $x_i$.
Second, we compute the one-sided auto-correlation of the principal component, and denote it by $A_{x_i}$. Then, we quantify the average presence of the fECG in the inputs $x_i$ by computing: $\bar{A}_x^{(f)}=\frac{1}{n_s}\sum_{i=1}^{n_s} A_{x_i}(\tau_i^{(f)})$, where $\tau_i^{(f)}$ denotes the periods of the fECG obtained from the doppler signals, and $n_s$ denotes the number of time segments in the evaluated recording.
Similarly, we compute $\bar{A}_x^{(m)}=\frac{1}{n_s}\sum_{i=1}^{n_s} A_{x_i}(\tau_i^{(m)})$, where $\tau_i^{(m)}$ denotes the periods of the mECG obtained from the thorax signals.
Finally, to quantify the relative presences of the signals, we compute the ratio $R_x=\frac{\bar{A}_x^{(f)}}{\bar{A}_x^{(m)}}$. 
We apply the same procedure to the codes $s'_i$, resulting in $R_{s'}$. When evaluating the baselines, we consider the signals obtained after the mECG cancellation as the counterparts of our code signals. 

In \Cref{tab:ResultsPerMethod}, we present the average ratios in the input, code, and baselines over all the subjects (see Appendix~\ref{sec:appendix_ecg} for results per subject).
We note that not all the subjects in the dataset include a noticeable fECG in the abdominal recordings. Therefore, we present results over subsets of top $k$ subjects showing highest average ratios $R_x$.
We see that our method significantly enhances the fECG with respect to the mixture, and it outperforms the tested baselines.

\section{Conclusion}\label{sec:conclusion}
In this paper, we present an autoencoder-based approach for single independent component recovery. The considered problem consists of observed data (mixture) generated from two independent components: one observed and the other hidden that needs to be recovered. 
We theoretically show that this ICA-related recovery problem can be accurately solved, in the sense that the hidden component is recovered up to an entropy-preserving function, by an autoencoder equipped with a discriminator.
In addition, we demonstrate the relevance of the problem and the performance of the proposed solution on several tasks, involving image manipulation, voice cloning, and fetal ECG extraction.

Future research will address the limitations of this work. 
\cref{lemma} assumes zero generalization loss, i.e., convergence to a global minimum, which is often not achieved in practice. In future work we plan to generalize this statement and assume bounded generalization loss. 
Another future direction will address noise robustness.
The current setting does not consider any noise, either in $t$ or in $x$, or even in $f$. For example, in the presence of noise, perfect reconstruction is undesired, and other losses need to be developed and used.

\section*{Acknowledgments and Disclosure of Funding}

We thank the reviewers for their important comments and suggestions. The work of OK and RT was supported by the European Union’s Horizon 2020 research and innovation programme under grant agreement No. 802735-ERC-DIFFOP. RT acknowledges the support of the Schmidt Career Advancement Chair in AI.





\bibliography{references}
\bibliographystyle{apalike}

\clearpage

\newpage
\appendix

\section{Colored MNIST experiment}\label{app:mnist}
In this experiment, we used a colored version of the MNIST handwritten image dataset, obtained by converting the images to RGB format and coloring each digit with an arbitrary color from $\{\text{red, green, blue} \}$.

We ran two experiments on this dataset. In the first one we considered the color as the condition. This setup perfectly meets the model assumptions, as each colored image was generated by choosing an arbitrary color at random ($t$) and coloring the original grayscale image ($s$).
In the second experiment we set the condition to be the digit label. This corresponds to a data generation process in which handwriting characteristics (e.g., line thickness, orientation) and color are independent of the digit label. While the color was indeed chosen independently of any other factor, independence of the handwriting characteristics and the digit label is debatable, as for example, orientation may depend on the specific digits (e.g., '1' is often written in a tilted fashion, while this is not the case for other digits). 

The condition $t$ was incorporated into decoder by modulating the feature maps before each convolutional layer.
The discriminator was trained using domain confusion loss. As a reconstruction term we used (pixel-wise) binary cross entropy.

Once the autoencoder was trained, we used it to manipulate the images by plugging to the decoder arbitrary condition and generating new data. 
Figure~\ref{fig:mnist} shows examples of reconstructions and manipulation for both experiments.
In the left panel (showing the results for condition=color) we can see that very high quality reconstruction and conversion were achieved, implying that the learned code did not contain color information, while preserving most of the information of the grayscale image, as desired.
The right panel (showing results for condition=digit label) displays similar results, although of somewhat worse conversion quality, as this setting does not fully fit the assumptions taken in this work. Yet, the code clearly captures most dominant characteristics of the handwriting.

\begin{figure*}[h]
  \centering
    \includegraphics[width=.45\textwidth]{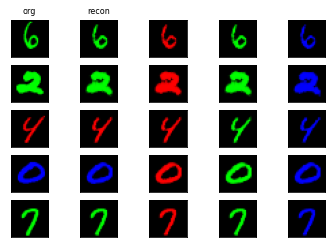}  
    \includegraphics[width=.475\textwidth]{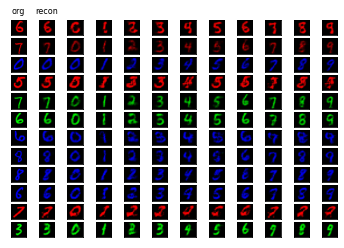}  
    \caption{The colored MNIST experiment, using the color (left) and digit label (right) as condition. In each of the plots, the leftmost column show the input $x$ to the encoder, and the next column shows the reconstruction. The remaining columns show conversion to each of the condition classes.}
    \label{fig:mnist}
\end{figure*}

\section{Image Domain Conversion}\label{app:cycle}
In this experiment we apply the proposed approach to some of the datasets introduced in~\cite{zhu2017unpaired}.
Here the condition is the domain (e.g., orange / apple). We use a combination of $\ell_1$ and SSIM loss for reconstruction and domain confusion for independence. In addition to reconstruction loss, we also use a GAN-like real/fake discriminator to slightly improve perceptual loss~\citep{blau2018perception}. Some results are shown in Figure~\ref{fig:cycle}.
While an interested reader might wonder why oranges are converted to red oranges rather than apples, we remark that as much as the condition specifies the type of fruit (orange / apple) throughout this dataset it also specifies its color (orange / red), which, by Ockham's razor, is a somewhat simpler description of the separation between the two domains. Therefore the image manipulation made by the model can be interpreted as a domain conversion.




\begin{figure}[h]
  \centering
    \includegraphics[width=.95\columnwidth]{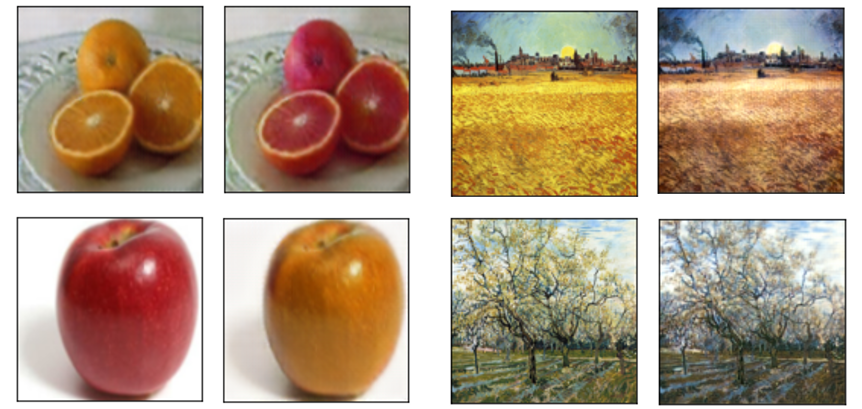}  
    \caption{Image Domain Conversion experiment. Left: Conversion results on the oranges2apples dataset. Right: Conversion from Cezanne to photo (up) and Van Goch to photo (down).} 
    \label{fig:cycle}
\end{figure}

\section{More Details on the Experiments}\label{app:tech}

\subsection{2D demonstration}
In this experiment we used MLP architectures for all networks, where each of the encoder, decoder and discriminator consisted of three hidden layers, each of size 64.
Identity and softplus activations were used for the linear and nonlinear mixing experiments, respectively.
The discriminator was regularized using r1 loss, computed every eight steps. 
The model was trained for 100 epochs on a dataset of 15,000 points.
To balance between the reconstruction and independence terms we used $\lambda=0.05$. The autoencoder optimizer was called every 5th step.
A Jupyter notebook running this demo is available at~\url{https://github.com/shaham-lab/disilv/blob/main/IMAGE/2D_demo.ipynb}.

\subsection{MNIST Experiment}
In this experiment each of the encoder and decoder consisted of three convolutional layers, of 16, 32 and 64 kernels and ReLU activations. The discriminator had a MLP architecture with 128 units in each layer.
The condition was incorporated into the decoder via modulation, utilizing a learnable lookup table of class embeddings of dimension 64.
In this experiment we also used an additional discriminator, trained in a GAN fashion to distinguish between real and generated images. During training this discriminator was also trained on swapped data, i.e., codes that were fed to the decoder with a random condition. This discriminator has three convolution layers and was trained with standard non-saturating GAN loss. 
The system was trained for 50 epochs, with $\lambda=0.01 $ and the real/fake discriminator loss term was added to the $\mathcal{L}_\text{AE}$ in~\eqref{eq:objecties} with coefficient of 0.001. 
A Jupyter notebook running this demo is available at~\url{https://github.com/shaham-lab/disilv/blob/main/IMAGE/colored_mnist_demo.ipynb}.

\subsection{Rotating Figures}
The images $x$ were of size $128\times128$ pixels and the conditions were of size $64\times$64. 
In the encoder images were passed through two downsampling blocks and two residual blocks. ResNet encoding models were applied to the condition and the images (separately), before the their feature maps were concatenated and passed through several additional ResNet blocks.
In the decoder conditions were downsampled once and passed through two residual blocks, before being concatenated to the codes and fed through two more residual blocks.
The decoder and discriminator have a similar architecture. We use $\ell_1$ as reconstruction loss.
The system was trained for 120 epochs on a dataset containing 10,000 instances, using $\lambda=1.05$, and the autoencoder was trained every 5th step.

\subsection{Image Domain Conversion}
In this experiment the encoder and decoder's architectures were inspired by the cycleGAN~\cite{zhu2017unpaired} ResNet generator architecture, splitting the generator to encoder and decoder. The decoder was enlarged with modulation layer before each convolutional layer. The class embeddings were of size 512.
As in the MNIST experiment, GAN-like real / fake discriminator was used here as well.
The system was trained for 200 epochs, on the datasets downloaded from the cyclegan official repository\footnote{\url{https://github.com/junyanz/pytorch-CycleGAN-and-pix2pix/blob/master/datasets/download_cyclegan_dataset.sh}}.
We used $\lambda_{ind}=\lambda_{rf} =0.1$ for both discriminators.
The autoencoder was trained every 5th step.
A Jupyter notebook running this demo is available at~\url{https://github.com/shaham-lab/disilv/blob/main/IMAGE/orange_apples_demo.ipynb}.

\subsection{Voice Conversion}
The encoder receives mel-filterbank features calculated from windows of 1024 samples with a 256 samples overlap, and outputs a latent representation of the source speech. The network is constructed from downsampling by factor 2 1D convolution layer with ReLU followed by 30 Jasper~\citep{li2019jasper} blocks, where each sub-block applies a 1D convolution, batch norm, ReLU, and dropout. All sub-blocks in a block have the same number of output channels which we set to 256. 
The decoder is also a convolutional neural network which receives as input the latent representation produced by encoder and the target speaker id as the condition. The condition was then mapped to a learnable embedding in \(\mathbb{R}^{64}\) concatenated to the encoder output by repeating it along the temporal dimension. The concatenated condition is passed through 1D convolution layer with stride 1 followed by a leaky-ReLU activation with a leakiness of 0.2 and 1D transposed convolution with stride 2 for upsampling to the original time dimension. 
The discriminator was trained using domain confusion loss. 
The system was trained for 2000 epochs, which took 8 days on a simple GTX 1080 GPU. We used $\lambda=1$, both optimizers were called every training step.

\subsubsection{Possible negative societal impact}

The ability to synthesize realistic audio using voice conversion can be exploited for malicious purposes, e.g., for voice spoofing, fake news, fraud, phishing, and harassment, to name but a few. Therefore, voice conversion should be deployed subject to ethical concerns. Specifically, the converted speech should be presented with an appropriate disclosure indicating that the synthetic speech was generated using a voice conversion technique. In addition, the conversion should be conditioned by obtaining consent from all the affiliated parties (both the source speaker and the target speaker). We acknowledge that this discussion is limited; for a broader discussion we refer the readers to \cite{ai2019high}.


\subsection{ECG Analysis}
\label{sec:appendix_ecg}
\subsubsection{Dataset}
\label{subsec:dataset_appendix}
The dataset consists of 60 entries from 39 voluntary pregnant women.
Each entry is composed of recordings from 27 ECG channels  and a synchronized recording from a trans-abdominal pulse-wave doppler (PWD) of the fetal's heart. The recordings' lengths vary from 7.5 seconds to 119.8 seconds (with average length of 30.6 seconds ± 20.6 seconds).
The ECG recordings were sampled by the TMSi Porti7 system with a frequency-sampling rate of 2KHz . The PWD recordings were acquired using the Philips iE33 ultrasound machine. The obtained video was converted into a 1D time series using a processing scheme based on envelope detection. The code for this processing scheme was provided as a Matlab-script by the authors of \cite{sulas2021non}. For convenience, we uploaded the obtained 1D time series after applying the provided Matlab-script, and it is available at~\url{https://github.com/shaham-lab/disilv/tree/main/ECG/Data}.


\subsubsection{Pre-proecssing and model implementation}
\label{subsec:model_appendix}
In the following we provide a detailed description of the pre-processing steps and the implementation of the model. For convenience, all the parameters and hyperparameters are summarized in \Cref{tab:ParamsList}.
The recording of subjects 1-20 were used for hyperparameters selection. These subjects were discarded in the objective evaluation reported in \Cref{tab:ResultsPerMethod}.

\paragraph{Pre-processing.}  
The raw ECG recordings were filtered by a median filter with a window length of $n_m=2,048$ (1 second) to remove the baseline drift. In addition, we apply a notch filter to remove the 50Hz powerline noise and a low-pass filter with a cut-off frequency of $F_c=125$Hz. Finally, we downsample the signal to frequency-sampling rate of $F_s=500$Hz.
The doppler signal was pre-processed using the script provided by the dataset's owners. No further operations were performed.

\paragraph{Implementation details.}
The encoder module $E(X)$ is implemented using a convolutional neural network (CNN): $\mathbb{R}^{n_a \times n_T} \rightarrow \mathbb{R}^{n_d \times n_T}$.
This choice of architecture is inspired by the architecture proposed by \cite{yildirim2018efficient} for the benefit of ECG compression, and it is described in detail in \Cref{tab:ArchEncoder}.

The implementation of the decoder module $D(S',T)$ is based on a deconvolutional neural network (dCNN): $\mathbb{R}^{(n_d+n_t) \times n_T} \rightarrow \mathbb{R}^{n_a \times n_T}$. This decoder is applied to the concatenation of the code signal and the thorax signal, where the concatenation is along the first dimension. The exact architecture is described in details in \Cref{tab:ArchDecoder}.

The discriminator is implemented via an additional CNN $g(T): \mathbb{R}^{n_t \times n_T} \rightarrow \mathbb{R}^{n_d \times n_T}$. $g(T)$ shares the same architecture as $E(X)$, except the first convolutional layer which has $n_t$ input channels rather than $n_a$.
The independence term is given by $\mathcal{I}( T,S') = \text{Ind}(g(T),S')$, where $\text{Ind}(x,y)$ is a scale-invariant version of the MSE loss function: $\text{Ind}(x,y)= \big|\big|\frac{|x|_e}{||x||_F}- \frac{|y|_e}{||y||_F} \big|\big|_F$, and $|\cdot|_e$ denotes an element-wise operation of absolute value.
We remark that other possibilities can be considered as well.

Lastly, the reconstruction module is simply implemented via the standard MSE loss: $\text{Recon}(x,y)= \big|\big| x -y  \big|\big|_F$.

\subsubsection{Training process}
We train a model for each subject. 
The training data is a collection of $n$ input-condition pairs $\{(x_i, t_i)\}_{i=1}^{n}$, where each input-condition pair $(x_i, t_i)$ is a time-segment that was selected from the ECG recording at a randomly drawn offset and $n$ is a hyperparameter indicating the number of randomly drawn training examples.
%
We use two optimizers that operate in an interleaved (adverserial-like) fashion. Specifically, for each update step of the second optimizer we perform $\beta$ update steps of the first optimizer, where $\beta=5$ is a hyperparameter. 
The first optimizer updates $g(T)$ and aims to maximize the dependency between the condition and the code. 
The second optimizer updates $E(X)$ and $D(S',T)$ and has two objectives -- minimizing the reconstruction loss while preventing the first optimizer from succeeding to maximize the dependency loss. Hence, encouraging the optimization process to converge to a ``condition-free" code. The proportion between these two objectives is controlled by the hyperparameter $\lambda=0.01$ which multiplies the second objective term.
The losses obtained by the two optimizers are denoted by $\mathcal{L}_\text{disc}$ and $\mathcal{L}_\text{AE}$ in \cref{eq:objecties}.

Both optimizers are implemented using the Adam algorithm \cite{kingma2014adam} with a fixed learning rate of $lr=10^{-4}$, $\beta=(0.9, 0.999)$ and a bach-size of $b=32$.

\begin{table}[t]
    \centering
    \caption{List of parameters and hyperparameters used in the ECG analysis. Parameters are listed in the upper part of the table, while hyperparameters are listed in the lower part of the table.}
    \label{tab:ParamsList}
    \begin{tabular}{|c|c|c|}
        \hline 
        \textbf{Notation} & \textbf{Description} & \textbf{Value}\\
        \hline \hline 
        $n_a$ & Number of abdominal channels & 24 \\ 
        $n_t$ & Number of thorax channels & 3 \\ 
        \hline
        $n_m$ & Window length of the median-filter& $2,000$ \\
        $F_c$ & Cut-off frequency & $5 \cdot 10^4$ \\
        $F_s$ & Frequency-sample & $500$ \\
        $n_d$ & Dimensioanlity of the code & 5 \\ 
        $n$ & Number  condition-pairs for training & $5 \cdot 10^4$ \\ 
        $b$ & Batch size & 32 \\ 
        $lr$ & Learning rate & $10^{-4}$ \\ 
        $\lambda$ & Objective independency factor & 0.01 \\ 
        $\beta$ & Interleaving independency factor & 5 \\
        \hline 
    \end{tabular}
\end{table}

\begin{table}[t]
    \centering
    \setlength{\tabcolsep}{1pt}
    \caption{Layers consisting the encoder $E(X)$ in the ECG analysis.}
    \label{tab:ArchEncoder}
    \begin{tabular}{|c|c|c|c|c|}
        \hline 
        \textbf{No} & \textbf{\specialcell{Layer\\ name}} & \textbf{\specialcell{No. of filters\\  $\times$ kernel size}} & \textbf{\specialcell{Activation \\function}} & \textbf{\specialcell{Output\\ size}}\\
        \hline \hline 
        1 & 1D Conv & $8\times 3$ & Tanh & $2000\times8$ \\ 
        2 & 1D Conv & $8\times 5$ & Tanh & $2000\times8$ \\ 
        3 & Batch Norm.  & - & - & $2000\times8$ \\ 
        4 & 1D Conv & $8\times 3$ & Tanh & $2000\times8$ \\ 
        5 & Batch Norm.  & - & - & $2000\times8$ \\ 
        6 & 1D Conv & $8\times 11$ & Tanh & $2000\times8$ \\ 
        7 & 1D Conv & $8\times 13$ & Tanh & $2000\times8$ \\ 
        8 & 1D Conv & $n_d\times3$ & Tanh & $2000\times n_d$ \\ 
        \hline 
    \end{tabular}
\end{table}

\begin{table}[t]
    \centering
    \setlength{\tabcolsep}{1pt}
    \caption{Layers consisting the decoder $D(S',T)$ in the ECG analysis. ``T.Conv" denotes a transposed convolution layer.}
    \label{tab:ArchDecoder}
    \begin{tabular}{|c|c|c|c|c|}
        \hline 
        \textbf{No} & \textbf{\specialcell{Layer\\ name}} & \textbf{\specialcell{No. of filters\\  $\times$ kernel size}} & \textbf{\specialcell{Activation \\function}} & \textbf{\specialcell{Output\\ size}}\\
        \hline \hline 
        1 & 1D T.Conv & $8\times 3$ & Tanh & $2000\times8$ \\ 
        2 & 1D T.Conv & $8\times 13$ & Tanh & $2000\times8$ \\ 
        3 & 1D T.Conv & $8\times 3$ & Tanh & $2000\times8$ \\ 
        4& 1D T.Conv & $8\times 5$ & Tanh & $2000\times8$ \\ 
        5 & 1D T.Conv & $n_a\times 3$ & Tanh & $2000\times n_a$ \\ 
    \hline 
    \end{tabular}
\end{table}

\subsubsection{Qualitative evaluation}
\label{subsubsec:qualitative_appendix}
Here, we describe in detail the procedure presented in \Cref{par:qualitative} in the paper.
First, we produce a code $s'_i$ for each input-condition pair $(x_i, t_i)$.
Then, we column-stack each matrix in the set  $\{s'_i\}_{i=1}^{1,000}$ and project the obtained set of vectors to a 3D space using principal component analysis (PCA). We repeat the same procedure for $\{x_i\}_{i=1}^{1,000}$.
We color the projected points in two manners: according to the fECG signal and according to the mECG signal.
The color of the $i$th sample representing the fECG (mECG) signal is computed as follows: $\{\text{mod}(i,\tau^{(f)})\}_{i=1}^{1,000}$ ($\{\text{mod}(i,\tau^{(m)})\}_{i=1}^{1,000}$), where $\tau^{(f)}$ ($\tau^{(m)}$) denotes the period of the fECG (mECG) obtained from the doppler signal (thorax recordings). 

\subsubsection{Additional results}
\label{subsec:objectiv_appendix}
The results presented in \Cref{par:quantitative} are averaged over subsets of subjects. In \Cref{tab:ResultsPerSubject} we present the results for each subject.

\begin{table}[t]
    \centering
    \tiny
    \caption{fECG extraction results for each subject. }
    \label{tab:ResultsPerSubject}
    \begin{tabular}{|c|c|c|c|c|c|c|c|}
        \hline 
        Subject & Input        & Ours         & ADALINE      & ESN& LMS& RLS   \\ 
        \hline  \hline 
         1 & 0.11 (0.10)  & 0.74 (1.63) & 0.51 (0.51)  & 0.05 (0.16)  & 0.00 (0.00)  & 0.20 (0.16)  \\
         2 & 0.13 (0.12)  & 0.51 (1.90) & 0.17 (0.58)  & 0.07 (0.12)  & 0.00 (0.00)  & -0.02 (0.10) \\
         3 & 0.10 (0.13)  & 0.49 (0.90) & 0.56 (0.59)  & 0.45 (0.33)  & 0.47 (0.33)  & 0.04 (0.14)  \\
         4 & 0.15 (0.22)  & 0.69 (0.68) & 0.10 (0.70)  & 0.15 (0.50)  & 0.07 (0.15)  & 0.08 (0.22)  \\
         5 & 7.89 (0.12)  & 0.74 (1.04) & -0.01 (0.69) & 0.22 (0.36)  & 0.07 (0.72)  & 0.04 (0.13)  \\
         6 & 0.11 (0.09)  & 0.58 (1.14) & 0.39 (0.54)  & 0.14 (0.22)  & 0.34 (0.45)  & 0.10 (0.15)  \\
         7 & 0.10 (0.11)  & 0.07 (0.99) & 1.04 (0.84)  & 0.32 (0.28)  & 0.60 (0.41)  & 0.09 (0.10)  \\
         8 & -0.00 (0.26) & 0.94 (0.69) & 0.26 (0.48)  & 0.18 (0.79)  & 0.00 (0.00)  & 0.17 (0.28)  \\
        10 & 0.10 (0.10)  & 4.41 (1.48) & 1.46 (2.01)  & 0.22 (0.25)  & 0.01 (0.02)  & 0.03 (0.12)  \\
        11 & 0.09 (0.12)  & 0.56 (1.73) & 0.68 (0.63)  & 0.11 (0.17)  & 0.00 (0.00)  & 0.15 (0.22)  \\
        13 & -0.05 (0.08) & 1.35 (0.50) & 1.08 (1.23)  & 0.03 (0.50)  & 0.00 (0.01)  & 0.58 (0.30)  \\
        14 & -0.03 (0.05) & 3.60 (3.18) & 8.82 (4.39)  & 1.85 (1.59)  & 3.17 (2.50)  & 0.90 (0.54)  \\
        15 & 0.13 (0.39)  & 0.52 (0.90) & 0.36 (0.67)  & 0.29 (1.04)  & 0.14 (0.82)  & 0.28 (0.22)  \\
        16 & 0.04 (0.20)  & 0.80 (1.57) & 0.20 (0.61)  & 0.09 (0.51)  & 0.00 (0.00)  & -0.03 (0.13) \\
        17 & 0.03 (0.23)  & 1.31 (0.45) & 0.29 (0.50)  & -0.06 (0.32) & 0.00 (0.00)  & 0.58 (0.54)  \\
        18 & 0.03 (0.05)  & 0.50 (0.82) & 0.58 (1.34)  & 0.04 (0.14)  & -0.01 (0.02) & 0.12 (0.19)  \\
        19 & 0.11 (0.08)  & 0.86 (0.94) & 1.03 (0.99)  & 0.19 (0.23)  & 0.00 (0.00)  & 0.22 (0.14)  \\
        20 & 0.17 (0.15)  & 0.86 (0.94) & 0.24 (0.52)  & 0.14 (0.38)  & 0.45 (0.54)  & 0.28 (0.39)  \\
        21 & 0.23 (0.15)  & 0.40 (0.45) & 0.28 (0.45)  & 0.24 (0.34)  & 0.36 (0.45)  & 0.44 (0.24)  \\
        22 & 1.17 (0.59)  & 1.41 (1.10) & 1.00 (1.65)  & 1.14 (0.69)  & 3.56 (2.09)  & 2.28 (2.57)  \\
        23 & 0.16 (0.24)  & 0.82 (1.31) & 0.17 (1.11)  & 0.23 (1.40)  & 0.01 (0.03)  & 0.03 (0.08)  \\
        24 & 0.05 (0.11)  & 0.92 (1.05) & 0.19 (0.74)  & 0.48 (0.65)  & 0.53 (0.35)  & 0.28 (0.25)  \\
        25 & -0.05 (0.08) & 0.36 (0.38) & 0.81 (0.60)  & 0.30 (0.36)  & 0.40 (0.53)  & -0.05 (0.20) \\
        26 & 0.09 (0.66)  & 0.67 (1.08) & 0.56 (0.69)  & 0.23 (0.40)  & 0.17 (0.34)  & 0.11 (0.14)  \\
        27 & 0.02 (0.12)  & 0.46 (0.71) & 0.34 (0.43)  & 0.10 (0.32)  & 0.20 (0.40)  & 0.09 (0.19)  \\
        28 & -0.01 (0.09) & 1.88 (1.80) & 1.55 (1.49)  & 0.02 (0.18)  & 0.00 (0.00)  & 0.15 (0.12)  \\
        29 & -0.07 (0.04) & 0.83 (0.90) & 1.09 (0.97)  & -0.10 (0.06) & 0.24 (0.04)  & 0.04 (0.18)  \\
        30 & 0.09 (0.11)  & 8.14 (3.21) & 0.95 (0.99)  & 0.43 (0.55)  & 0.83 (0.97)  & 0.34 (0.30)  \\
        31 & 0.10 (0.13)  & 1.06 (0.60) & 0.33 (0.33)  & 0.05 (0.24)  & 0.01 (0.00)  & 0.01 (0.20)  \\
        32 & 0.03 (0.22)  & 1.27 (1.04) & 0.25 (0.24)  & -0.06 (0.22) & 0.00 (0.00)  & -0.05 (0.23) \\
        33 & 0.03 (0.20)  & 0.30 (0.35) & 0.35 (0.24)  & 0.02 (0.14)  & 0.00 (0.00)  & -0.09 (0.13) \\
        37 & 0.09 (0.18)  & 0.82 (0.70) & 0.02 (0.72)  & 0.21 (0.40)  & 0.00 (0.00)  & 0.15 (0.11)  \\
        39 & 0.13 (0.17)  & 9.57 (3.24) & 9.16 (1.67)  & 0.12 (0.17)  & 0.05 (0.02)  & 0.48 (0.24)  \\
        41 & 0.08 (0.08)  & 5.66 (7.76) & 4.28 (3.49)  & 0.13 (0.34)  & 0.00 (0.00)  & 0.32 (0.51)  \\
        43 & -0.07 (0.04) & 1.21 (0.80) & 1.61 (1.00)  & 0.29 (0.38)  & 0.56 (0.37)  & 0.35 (0.26)  \\
        44 & -0.06 (0.08) & 1.42 (1.03) & 1.09 (1.36)  & 0.00 (0.80)  & 0.09 (0.09)  & 0.27 (0.19)  \\
        45 & 1.69 (0.43)  & 4.30 (4.73) & 0.56 (1.57)  & 0.95 (1.81)  & 1.02 (0.68)  & 0.77 (0.77)  \\
        46 & 0.18 (0.11)  & 3.47 (2.77) & 6.45 (5.37)  & 3.62 (1.68)  & 4.45 (4.05)  & 0.17 (0.15)  \\
        48 & 0.06 (0.21)  & 3.40 (1.68) & 3.12 (2.79)  & 2.37 (1.44)  & 0.03 (0.02)  & 0.24 (0.55)  \\
        50 & 0.13 (0.27)  & 0.50 (0.79) & 0.52 (0.65)  & 0.21 (0.32)  & 0.00 (0.00)  & 0.01 (0.08)  \\
        51 & 0.13 (0.24)  & 0.24 (0.48) & 3.59 (2.63)  & 0.13 (0.27)  & 0.00 (0.00)  & 0.65 (0.26)  \\
        55 & 0.15 (0.16)  & 4.65 (5.32) & 2.15 (1.78)  & 0.10 (0.18)  & 0.01 (0.00)  & 0.04 (0.15)  \\
        56 & 0.14 (0.13)  & 4.24 (1.60) & 0.77 (0.66)  & 0.04 (0.23)  & 0.04 (0.11)  & 0.08 (0.15)  \\
        58 & 0.09 (0.31)  & 6.29 (4.99) & -0.01 (0.67) & 0.15 (0.53)  & 0.17 (0.62)  & 0.45 (0.33)  \\
        59 & 0.09 (0.09)  & 0.79 (0.88) & 0.34 (0.71)  & 0.14 (0.23)  & 0.00 (0.00)  & 0.35 (0.17)  \\
        \hline
    \end{tabular}
\end{table}

\subsubsection{Additional comments}

We looked for a dataset that contains: (1) abdominal recordings, (2) chest (thorax) recordings, and (3) ground-truth (GT) that can be used for quantitative evaluation. 
For this purpose, we reviewed all the datasets from Section 7 in~\cite{behar2016practical}:
\begin{itemize}
\item DDB and NIFECGDB:  these two datasets do not have GT.
\item ADFECGDB:  in this dataset, there are no chest recordings.
\item PCDB and ADFECGDB:  these datasets do not include chest recordings.
\item FECGSYNDB:  seemingly, this dataset admits all the requirements. However, it is a synthetic dataset, and we were looking for a real-world dataset.
\end{itemize}
The only real-world datasets that include chest recordings are DDB and NIFECGDB. DDB includes only a single recording of a single subject, and therefore, we focused on NIFECGDB. This dataset was used to objectively evaluate ESN (one of the considered baselines) in [2] using expert annotations. We contacted the authors of~\cite{behar2013echo} and asked them to share their annotations. Unfortunately, the authors could not share the annotations, but they kindly referred us to use the NIFEADB dataset~\cite{pani2020ninfea}. This recently published dataset fits our purposes, and therefore, we chose to use it in our experiments.

We note that the GT in NIFEADB is not given as expert annotations of the fetal QRS complexes as in ADFECGDB, PCDB, and the proprietary annotations from NIFECGDB. In NIFEADB, the GT is extracted from the Doppler signal of the fetal heart, making the use of commonly accepted evaluation metrics proposed in~\cite{behar2016practical} impossible for the following reasons:
\begin{itemize}
\item Fetal HR measures (listed in the first part of table 5 in~\cite{behar2016practical}: Se,PPV, F1,etc) -- these measures assume that the GT includes the locations of the fetal QRS complexes.
\item Morphological analysis (listed in the second part of table 5 in~\cite{behar2016practical}: SNR,FQT,TQRS) -- these measures assume that the GT includes the fECG signal, which is available only in simulations and in invasive procedures.
\end{itemize}
Therefore, we used a quantitative evaluation metric that quantifies the enhancement of the fECG and the suppression of the mECG based on the doppler GT.

We remark that the lack of publicly-available reference datasets, which could be used to benchmark different algorithms, was the main motivation for the curation of the NIFEADB (see the abstract in~\cite{pani2020ninfea}). However, establishing such benchmarks and gold standards is still an ongoing effort, and, to the best of our knowledge, there is no definitive gold standard criterion available to date.

\end{document}